%% file: manuscript.tex
\documentclass[sigconf]{aamas}  

\AtBeginDocument{%
  \providecommand\BibTeX{{%
    \normalfont B\kern-0.5em{\scshape i\kern-0.25em b}\kern-0.8em\TeX}}}

\usepackage{nicefrac}
\usepackage{microtype} 
\usepackage{float,array}
\usepackage{amssymb,amsmath,bm,amsthm}
\usepackage[caption=false,font=scriptsize]{subfig}
\usepackage[ruled,vlined]{algorithm2e}
\usepackage{enumerate}
\usepackage{color,colortbl}
\usepackage{epstopdf}
\usepackage{booktabs, multirow}
\definecolor{darkgreen}{rgb}{0, 0.5, 0}
\definecolor{red}{rgb}{1, 0, 0}
\definecolor{purple}{rgb}{0.5, 0, 0.5}
\usepackage{chngpage}

\usepackage{mfirstuc}
\usepackage{mathtools}
\usepackage{lettrine}
\usepackage{epstopdf}

\usepackage{flushend}

\DeclareMathOperator*{\argmin}{argmin}

\newcommand\algoname{META}

\newcommand\ie{\textit{i.e.}}
\newcommand\eg{\textit{e.g.}}
\newcommand\st{\textit{s.t.}}
\newcommand\wrt{\textit{w.r.t.}}

\newcommand\doubleE{\mathbb{E}}
\newcommand\doubleP{\mathbb{P}}

\newcommand\scriptS{\mathcal{S}}
\newcommand\scriptE{\mathcal{E}}

\newtheorem*{theorem*}{Theorem}
\newtheorem{fact}{Fact}
\newtheorem*{proposition*}{Proposition}
\newtheorem*{corollary*}{Corollary}

\setcopyright{ifaamas}  
\copyrightyear{2020} 
\acmYear{2020} 
\acmDOI{} 
\acmPrice{} 
\acmISBN{} 
\acmConference[AAMAS'20]{Proc.\@ of the 19th International Conference on Autonomous Agents and Multiagent Systems (AAMAS 2020)}{May 9--13, 2020}{Auckland, New Zealand}{B.~An, N.~Yorke-Smith, A.~El~Fallah~Seghrouchni, G.~Sukthankar (eds.)}  

\begin{document}
\title{\algoname{}-Learning State-based Eligibility Traces for More Sample-Efficient Policy Evaluation}
\titlenote{The first three authors contributed equally to this paper. This paper has a description video at: \url{https://www.youtube.com/watch?v=3Ud8Ils1_mo}}


\author{Mingde Zhao$^{*}$}
\affiliation{%
  \institution{Mila, McGill University}
}
\email{mingde.zhao@mail.mcgill.ca}

\author{Sitao Luan$^{*}$}
\affiliation{%
  \institution{Mila, McGill University}
}
\email{sitao.luan@mail.mcgill.ca}

\author{Ian Porada$^{*}$}
\affiliation{%
  \institution{Mila, McGill University}
}
\email{ian.porada@mail.mcgill.ca}

\author{Xiao-Wen Chang}
\affiliation{%
  \institution{McGill University}
}
\email{chang@cs.mcgill.ca}

\author{Doina Precup}
\affiliation{%
  \institution{DeepMind, Mila, McGill University}
}
\email{dprecup@cs.mcgill.ca}

\begin{abstract}
Temporal-Difference (TD) learning is a standard and very successful reinforcement learning approach, at the core of both algorithms that learn the value of a given policy, as well as algorithms which learn how to improve policies. TD-learning with eligibility traces provides a way to boost sample efficiency by temporal credit assignment, \ie{} deciding which portion of a reward should be assigned to predecessor states that occurred at different previous times, controlled by a parameter $\lambda$. However, tuning this parameter can be time-consuming, and not tuning it can lead to inefficient learning. For better sample efficiency of TD-learning, we propose a meta-learning method for adjusting the eligibility trace parameter, in a state-dependent manner. The adaptation is achieved with the help of auxiliary learners that learn distributional information about the update targets online, incurring roughly the same computational complexity per step as the usual value learner. Our approach can be used both in on-policy and off-policy learning. We prove that, under some assumptions, the proposed method improves the overall quality of the update targets, by minimizing the overall target error. This method can be viewed as a plugin to assist prediction with function approximation by meta-learning feature (observation)-based $\lambda$ online, or even in the control case to assist policy improvement. Our empirical evaluation demonstrates significant performance improvements, as well as improved robustness of the proposed algorithm to learning rate variation.
\end{abstract}


\maketitle

\section{Introduction}\label{sec:introduction}

Eligibility trace-based policy evaluation (prediction) methods, \eg, TD($\lambda$), use geometric sequences, controlled by a parameter $\lambda$, to weight multi-step returns and assemble compound update targets \cite{sutton2018reinforcement}. Given a properly set $\lambda$, using $\lambda$-returns as update targets lowers the sample complexity (\eg, the number of steps to achieve certain precision of policy evaluation) or equivalently, improves the learning speed and accuracy. 

Sample complexity in Reinforcement Learning (RL) is sensitive to the choice of the hyperparameters \cite{sutton2018reinforcement,white2016greedy}. To address this, meta-learning has been proposed as an approach for adapting the learning rates \cite{dabney2012adaptive}. However, the design of principle approaches and maintenance of low computational complexity yield difficulties to tackle the problem \cite{kearns2000bias,schapire1996worst}. Some Bayesian offline method has been proposed to address this problem \cite{downey2010bayesian}. Some methods have been proposed for online meta-learning, with high extra computational complexities that are intolerable for practical use \cite{mann2016adaptive}. Some methods seek to create replacements of TD($\lambda$) with better properties, mixing only the Monte-Carlo return and $1$-step TD return \cite{penedones2019adaptive}. To summarize, a principled method for adapting $\lambda$s online and efficiently is in need.

TD($\lambda$) with different $\lambda$ values for different states has been proposed as a more general formulation of trace-based prediction methods. While preserving good mathematical properties such as convergence to fixed points, this generalization also unlocks significantly more degrees of freedom than only adapting a constant $\lambda$ for every state.
It is intuitively clear that using state-based values of $\lambda$ provides more flexibility than using a constant for all states. 
\cite{white2016greedy} investigated the use of state-based $\lambda$s, while outperforming constant $\lambda$ values on some prediction tasks. The authors implicitly conveyed the idea that better update targets lead to better sample efficiency, \ie, update targets with smaller Mean Squared Error (MSE) lead to smaller MSE in learned values. Their proposed online adaptation is achieved via efficient incremental estimation of statistics about the return targets, gathered by some auxiliary learners. Yet, such method does not seek to improve the overall sample efficiency, because the meta-objectives does not align with the overall target quality.

The contribution of this paper is a principled method for meta-learning state- or feature-based parametric $\lambda$s \footnote{State-based for tabular case, and feature-based for function approximation case.} which aims directly at the sample efficiency. Under some assumptions, the method has the following properties:
\begin{enumerate}
\item Meta-learns online and uses only incremental computations, incurring the same computational complexity as usual value-based eligibility-trace algorithms, such as TD($\lambda$).
\item Optimizes the overall quality of the update targets.
\item Works in off-policy cases.
\item Works with function approximation.
\item Works with adaptive learning rate.
\end{enumerate}

\section{Preliminaries}\label{sec:preliminaries}
TD($\lambda$) \cite{Singh1996} uses geometric sequences of weights controlled by a parameter $\lambda \in [0,1]$ to compose a compound return as the update target, which is called the $\lambda$-return. When used online, the updates towards the $\lambda$-return can be \textit{approximated} with incremental updates using buffer vectors called the ``eligibility traces'' with linear spacetime complexity.


\input{tab_notations.tex}

\subsection{The Trace Adaptation Problem}
We aim to find an online meta-learning method for the adaptation of state- or feature-based $\lambda$s to achieve higher sample efficiency (faster and more accurate prediction in terms of MSE of the value estimate) in unknown environments.


\subsection{Background Knowledge}
Before everything, we first present all the notations in Table \ref{tab:notations}.

\begin{definition}[Update Target]
When an agent is conducting policy evaluation, the \textbf{update target} (or \textbf{target}) is a random variable towards whose observed value the agent updates its value estimates.
\end{definition}

Fixed-step update targets are also random variables. For example, the update target for $1$-step TD is $r_{t+1} + \gamma_{t+1} V(s_{t+1})$ and the update target for TD($\bm{\lambda}$) with state-based $\lambda$s is the (generalized) $\bm{\lambda}$-return, as defined below.

\begin{definition}[$\bm{\lambda}$-return]
The \textbf{generalized state-based $\bm{\lambda}$-return} $G_t^{\bm{\lambda}}$, where $\bm{\lambda} \equiv {[\lambda_1, \dots, \lambda_i \equiv \lambda(s_i), \dots, \lambda_{|\scriptS|}]}^T$, for one state $s_t$ in a trajectory $\tau$ is recursively defined as
\begin{equation}\nonumber
G_t^{\bm{\lambda}} \equiv G^{\bm{\lambda}}(s_t) = r_{t+1} + \gamma_{t+1} [(1 - \lambda_{t+1})V(s_{t+1}) + \lambda_{t+1}G_{t+1}^{\bm{\lambda}}]
\end{equation}
where $G_t^{\bm{\lambda}} = 0$ for $t > |\tau|$.
\end{definition}
Prediction using the generalized $\bm{\lambda}$-return has well-defined fixed points \cite{white2016greedy}. However, when using trace-based updates online, such convergence can only be achieved with the true online algorithms \cite{hasselt2014true,seijen2015true}. With the equivalence provided by the true online methods, we will also have the full control of the bias-variance tradeoff of the update targets via $\bm{\lambda}$ even if learning online.

\par

The quality of the update targets, which we aim to enhance, has important connections to the quality of the learned value function \cite{singh1997analytical}, which we ultimately pursue in policy evaluation tasks.

\begin{definition}[Overall Value Error \& State Value Error]
Given the true value $\bm{v}$ and an estimate $\bm{V}$ of target policy $\pi$, the \textbf{overall value error} for $\bm{V}$ is defined as:
\begin{equation}
\nonumber
J(\bm{V}) \equiv 1/2 \cdot {\| D^{1/2} \cdot (\bm{V} - \bm{v}) \|}_2^2
\end{equation}
where 
\begin{equation} \label{eq:Ddef}
D \equiv diag(d_\pi(s_1), d_\pi(s_2), \cdots, d_\pi(s_{|S|}))
\end{equation}
For a particular state $s$, the \textbf{state value error} is defined as
$$J(V(s)) \equiv 1/2 \cdot (V(s) - v(s))^2$$
\end{definition}

The weights favor the states that will be met with higher frequency under policy $\pi$. We often use the overall value error to evaluate the performance of value learners in prediction tasks \cite{singh1997analytical}.

\begin{definition}[Overall Target Error \& State Target Error]
Given $\bm{v}$ and the collection of the update targets $\hat{\bm{G}}$ for all states, the \textbf{overall mean squared target error} or \textbf{overall target error} for $\hat{\bm{G}}$ is defined as:
\begin{equation}
\nonumber
J(\hat{\bm{G}}) \equiv 1/2 \cdot {\| D^{1/2} \cdot (\doubleE[\hat{\bm{G}}] - \bm{v}) \|}_2^2
\end{equation}
where $D$ is defined in \eqref{eq:Ddef}.
For a particular state $s$, the \textbf{state target error} or \textbf{target error} is defined as
$$J(\hat{G}(s)) \equiv 1/2 \cdot (\hat{G}(s) - v(s))^2$$
\end{definition}

Updates are never conducted for the terminal states. Thus, the target error and value error for terminal states should be set $0$, as these states are always identifiable from the terminal signals. The errors of the values and the targets are strongly connected.

\begin{proposition}
Given suitable learning rates, value estimates using targets with lower overall target error asymptotically achieve lower overall value error.
\end{proposition}

Though it can easily be proved, the conclusion is very powerful: sample efficiency can be enhanced by using better update targets, which in the trace-based prediction means optimizing the difference between the update target and the true value function. This is the basis for the $\lambda$-greedy algorithm which we are about to discuss as well as our proposed method.

\par

\subsection{$\lambda$-Greedy \cite{white2016greedy}: An Existing Work}
$\lambda$-greedy is a meta-learning method that can achieve online adaptation of state-based $\lambda$s with the help of auxiliary learners that learn additional statistics about the returns. 
The idea is to minimize the error between a pseudo target $\tilde{G}(s_t)$ and the true value $\bm{v}(s_t)$, where the pseudo target is defined as:
$$\tilde{G}(s_t) \equiv \tilde{G}_t \equiv r_{t+1} + \gamma_{t+1} [(1 - \lambda_{t+1})V(s_{t+1}) + \lambda_{t+1} G_{t+1}]$$
where $\lambda_{t+1} \in [0, 1]$ and $\lambda_k = 1, \forall k \geq t + 2$. 
With this we can find that $\tilde{J}(s_t) \equiv \doubleE[{(\tilde{G}_t - \doubleE[G_t])}^2]$ is a function of only $\lambda_{t+1}$ (given the value estimate $V(s_{t+1})$). The greedy objective corresponds to minimizing the error of the pseudo target $\tilde{G}_t$:

\begin{fact}[\cite{white2016greedy}]\label{prop:greedy_minimizer}
Let $t$ be the current timestep and $s_t$ be the current state. If the agent takes action at $s_t$ \st{} it will transition into $s_{t+1}$ at $t+1$.
Given the pseudo update target $\tilde{G}_t$ of $s_t$, the minimizer $\lambda_{t+1}^*$ of the target error of the state $\tilde{J}(s_t) \equiv \doubleE[{(\tilde{G}_t - \doubleE[G_t])}^2]$ \wrt{} $\lambda_{t+1}$ is:
\begin{equation}\label{eq:white_argmin}
\lambda_{t+1}^* = \frac{(V(s_{t+1}) - \doubleE[G_{t+1}])^2}{(\doubleE[V(s_{t+1}) - G_{t+1}])^2 + Var[G_{t+1}]}
\end{equation}
where $G_{t+1}$ is the Monte Carlo return. 
\end{fact}

The adaptation of $\bm{\lambda}$ in $\lambda$-greedy needs auxiliary learners, that run in parallel with the value learner, for the additional distributional information needed, more specifically the expectation and the variance of the MC return, preferably in an incremental manner. The solutions for learning these have been contributed in \cite{white2016greedy,sherstan2018directly}. These methods learn the variance of $\bm{\lambda}$-return in the same way TD methods learn the value function, however with different ``rewards'' and ``discount factors'' for each state, that can be easily obtained from the known information without incurring new interactions with the environment.

$\lambda$-greedy gives strong boost for sample efficiency in some prediction tasks. However, there are two reasons that $\lambda$-greedy has much space to be improved. The first is that the pseudo target $\tilde{G}_t$ used for optimization is not actually the target used in TD($\bm{\lambda}$) algorithms: we will show that it is rather a compromise for a harder optimization problem; The second is that setting the $\lambda$s to the minimizers does not help the overall quality of the update target: the update targets for every state is controlled by the whole $\bm{\lambda}$, thus unbounded changes of $\bm{\lambda}$ for one state will inevitably affect the other states as well as the overall target error.

From the next section, we build upon the mindset provided in \cite{white2016greedy} to propose our method \algoname{}.

\section{Meta Eligibility Trace Adaptation}\label{sec:\algoname{}}
In this section, we propose our method \algoname{}, whose goal is to find an off-policy compatible way for optimizing the overall target error while keeping all the computations online and incremental. Our approach is intuitively straight-forward: optimizing overall target error via optimizing the ``true'' target error for every state, \ie, the errors of $\bm{\lambda}$-returns, properly.

We first investigate how the goal of optimizing overall target error can be achieved online. A key to solving this problem is to acknowledge that the states that the agent meets when carrying out the policy $\pi$ follows the distribution of $d_\pi$. Since the overall target error is a weighted mix of the state target errors according to $d_\pi$, this infers the possibility of decomposing the optimization of the overall target error to the optimizations of the state target errors, for which we optimize each state target error to the same extent and then the state distribution could mix our sub-optimizations together to form a joint optimization of the overall target error. We develop the following theorem to construct this process.

\begin{theorem}\label{thm:nongreedy}
Given an MDP, target and behavior policies $\pi$ and $b$, let $D$ be diagonalized state frequencies $d_\pi(s_1), \cdots, d_\pi(s_{|\scriptS|})$ and $\hat{\bm{G}} \equiv [G_{s_1}(\bm{\lambda}), \cdots, G_{s_{|\scriptS|}}(\bm{\lambda})]^T$ be the vector assembling the state update targets, in which the targets are all parameterized by a shared parameter vector $\bm{\lambda}$. The gradient of the overall target error $J(\hat{\bm{G}}, \scriptS) \equiv 1/2 \cdot \doubleE_\pi \left[ {\| D^{1/2} \cdot (\hat{\bm{G}} - \bm{v})] \|}_2^2 \right]$ can be assembled from $1$-step gradients on the target error $J(\hat{G}_s, s) \equiv 1/2 \cdot (\hat{G}_s(\bm{\lambda}) - v_s)^2$ of update target $\hat{G}_s$ for every state $s$ the agent is in when acting upon behavior policy $b$, where weights are the cumulative product $\rho_{acc}$ of importance sampling ratios from the beginning of the episode until $s$. Specifically:
$$\nabla_{\bm{\lambda}} J(\hat{\bm{G}}, \scriptS) \propto \sum_{s \sim b}{\rho_{acc} \cdot \nabla_{\bm{\lambda}} J(\hat{G}_s, s)}$$
where $b$ is the behavior policy.
\end{theorem}


\begin{proof}
According to the definition of overall target error,
\begin{equation}
\begin{aligned}
J({\bm{{\bm{\lambda}}}}) \equiv \sum_{s \in \scriptS}{d_\pi(s) \cdot J_s(\bm{\lambda}))} = \sum_{s \in \scriptS}{d_\pi(s) \cdot \doubleE[G^{\bm{{\bm{\lambda}}}}(s) - v(s)]^2}
\end{aligned}\nonumber
\end{equation}
If we take the gradient \wrt{} ${\bm{\lambda}}^{(t+1)}$ we can see that:

\begin{equation}
\small
\begin{aligned}
& \nabla J(\hat{\bm{G}}, \scriptS) = \sum_{s \in \scriptS}{d_\pi(s) \cdot \nabla J(\hat{G}_s, s)}\\
& \text{push the gradient inside}\\
& = \sum_{s \in \scriptS}{\sum_{k=0}^{\infty}{\doubleP\{s_0 \to s, k, \pi, s_0 \sim d(s_0) \} \cdot \nabla J(\hat{G}_s, s)}}\\
& \text{$\doubleP\{\cdots\}$ is the prob. of $s_0 \to \cdots \to s$ in $k$ steps,}\\
& \text{$s_0$ is sampled from the starting distribution $d(s_0)$.}\\
& = \sum_{s \in \scriptS}{\sum_{k=0}^{\infty}{\sum_{\tau}{\doubleP\{s_0 \xrightarrow{\tau} s, k, \pi, s_0 \sim d(s_0) \} \cdot \nabla J(\hat{G}_s, s)}}}\\
& \text{$\tau=s_0, a_0, s_1, a_1,\dots,a_{k-1},s_k=s$ is a trajectory starting from $s_0$,}\\
& \text{following $\pi$ and transitioning to $s$ in $k$ steps.}\\
& = \sum_{s \in \scriptS}{\sum_{k=0}^{\infty}{\sum_{\tau}{d(s_0)\cdots p(\tau_{k-1}, a_{k-1}, s)\pi(a_{k-1}|\tau_{k-1}) \cdot \nabla J(\hat{G}_s, s)}}}\\
& \text{$\tau_i$ is the $i+1$-th state of $\tau$ and $p(s, a, s^{'})$ is the prob. of $s \xrightarrow{a} s^{'}$ in the MDP}\\
& = \sum_{s \in \scriptS}{\sum_{k=0}^{\infty}{\sum_{\tau}{d(s_0) \cdots p(\tau_{k-1}, a_{k-1}, s) \frac{\pi(a_{k-1}|\tau_{k-1})}{b(a_{k-1}|\tau_{k-1})} b(a_{k-1}|\tau_{k-1}) \cdot \nabla J(\hat{G}_s, s)}}}\\
& \text{for the convenience of injecting importance sampling ratios}\\
& = \sum_{s \in \scriptS}{\sum_{k=0}^{\infty}{\sum_{\tau}{d(s_0) \cdots p(\tau_{k-1}, a_{k-1}, s) \rho_{k-1} b(a_{k-1}|\tau_{k-1}) \cdot \nabla_{\bm{\lambda}}J_s(\bm{\lambda}))}}}\\
& \text{$\rho_{i} \equiv \frac{\pi(a_i | \tau_i)}{b(a_i | \tau_i)}$ is the importance sampling (IS) ratio}\\
& = \sum_{s \in \scriptS}{\sum_{k=0}^{\infty}{\sum_{\tau}{\rho_{0:k-1} \cdot d(s_0) \cdots p(\tau_{k-1}, a_{k-1}, s) b(a_{k-1}|\tau_{k-1}) \cdot \nabla J(\hat{G}_s, s)}}}\\
& \text{$\rho_{0:i} \equiv \prod_{v=0}^{i}{\rho_{v}}$ is the product of IS ratios of $\tau$ from $\tau_0$ to $\tau_i$}\\
& = {\sum_{k=0}^{\infty}{\sum_{\tau} \rho_{0:k-1} \sum_{s \in \scriptS} { \cdot d(s_0) \cdots p(\tau_{k-1}, a_{k-1}, s) b(a_{k-1}|\tau_{k-1}) \cdot \nabla J(\hat{G}_s, s)}}} \\
& = \doubleE{_{\tau \sim b}\left[] \rho_{\tau} \nabla J(\hat{G}_s, s)\right]} \approx \sum_{s \sim b}{\rho_{acc} \cdot \nabla J(\hat{G}_s, s)}\\
& \text{equivalent to summing over the experienced states under $b$}
\end{aligned}\nonumber\normalsize
\end{equation}
\end{proof}

The on-policy case is easily proved with $b=\pi$. The theorem applies for general parametric update targets including $\bm{\lambda}$-return. Optimizing $\bm{\lambda}$ for each state will inevitably affect the other states, \ie, decreasing target error for one state may increase the others. The theorem shows if we can do gradient descent on the target error of the states according to $\bm{d}_\pi$, we can achieve optimization on the overall target error, assuming the value function is changing slowly. The problem left for us is to find a way to calculate or approximate the gradients of $\bm{\lambda}$ for the state target errors.

\begin{figure*}
\centering
\includegraphics[width=0.8\textwidth]{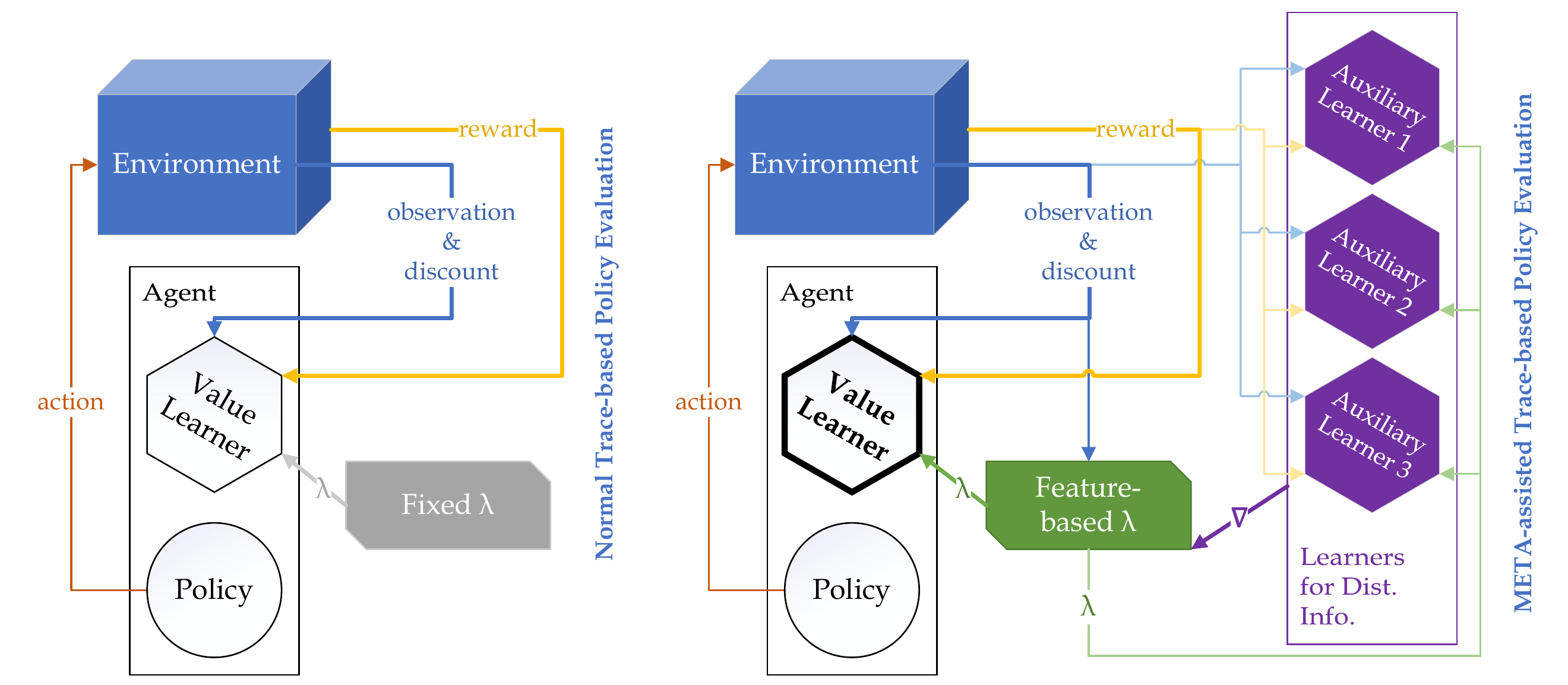}
\caption{Mechanisms for \algoname{}-assisted trace-based policy evaluation: the auxiliary learners learn the distributional information in parallel to the value learner and provide the approximated gradient for the adjustments of $\bm{\lambda}$.}
\label{fig:\algoname{}}
\end{figure*}

The exact computation of this gradient is infeasible in the online setting: in the state-based $\lambda$ setting, the $\bm{\lambda}$-return for every state is interdependent on every $\lambda$ of every state. These states are unknown before observation. However, we propose a method to estimate this gradient by estimating the partial derivatives in the dimensions of the gradient vector, which are further estimated online using auxiliary learners that estimates the distributional information of the update targets. The method can be interpreted as optimizing a bias-variance tradeoff.

\begin{proposition}\label{prop:objective}
Let $t$ be the current timestep and $s_t$ be the current state. The agent takes action $a_t$ at $s_t$ and will transition into $s_{t+1}$ at $t+1$ while receiving reward $r_{t+1}$. Suppose that $r_{t+1}$ and $G_{t+1}^{\bm{\lambda}}$ are uncorrelated, given the update target $G_t^{\bm{{\bm{\lambda}}}}$ for state $s_t$, the (semi)-partial derivative of the target error $J_{s_t}(\bm{\lambda}) \equiv \doubleE[(G_t^{\bm{{\bm{\lambda}}}} - \doubleE[G_t])^2]$  of the state $s_t$ \wrt{} $\lambda_{t+1} \equiv \lambda(s_{t+1})$ is:
\begin{equation}
    \begin{aligned}
    \frac{\partial J_{s_t}(\bm{\lambda})}{\partial \lambda_{t+1}} = & \gamma_{t+1}^2 [\lambda_{t+1} \left[ (V(s_{t+1}) - \doubleE[G_{t+1}^{\bm{{\bm{\lambda}}}}])^2 + Var[G_{t+1}^{\bm{{\bm{\lambda}}}}] \right]\\
    & + (\doubleE[G_{t+1}^{\bm{{\bm{\lambda}}}}] - V(s_{t+1}))
    (\doubleE[G_{t+1}] - V(s_{t+1}))]
    \end{aligned}\nonumber
\end{equation}
And its minimizer \wrt{} $\lambda_{t+1}$ is:
\begin{equation}
\begin{aligned}
& \argmin_{\lambda_{t+1}}{J_{s_t}(\bm{\lambda})} = \frac{
(V(s_{t+1}) - \doubleE[G_{t+1}^{\bm{{\bm{\lambda}}}}])
    (V(s_{t+1}) - \doubleE[G_{t+1}])
}
{(V(s_{t+1}) - \doubleE[G_{t+1}^{\bm{{\bm{\lambda}}}}])^2 + Var[G_{t+1}^{\bm{{\bm{\lambda}}}}]}
\end{aligned}\nonumber
\end{equation}
\end{proposition}
\begin{proof}
\begin{equation}
\begin{aligned}
J_{s_t}(\bm{\lambda}) \equiv \doubleE[(G_t^{\bm{{\bm{\lambda}}}} - \doubleE[G_t])^2] & = \doubleE^2[G_t^{\bm{{\bm{\lambda}}}} - G_t] + Var[G_t^{\bm{\lambda}}]
\end{aligned}
\nonumber
\end{equation}

\begin{equation}
\begin{aligned}
& \doubleE[G_t^{\bm{{\bm{\lambda}}}} - G_t]\\
& = \doubleE[r_{t+1}+\gamma_{t+1}((1 - \lambda_{t+1})V(s_{t+1})+\lambda_{t+1}G_{t+1}^{\bm{\lambda}})\\
& -(r_{t+1}+\gamma_{t+1}G_{t+1})]\\
& = \gamma_{t+1} (1 -\lambda_{t+1})V(s_{t+1}) + \gamma_{t+1} \lambda_{t+1} \doubleE[G_{t+1}^{\bm{\lambda}}] - \gamma_{t+1} \doubleE[G_{t+1}]
\end{aligned}
\nonumber
\end{equation}

\begin{equation}
\begin{aligned}
& Var[G_t^{\bm{\lambda}}] = Var[r_{t+1} + \gamma_{t+1} [(1 - \lambda_{t+1})V(s_{t+1}) + \lambda_{t+1}G_{t+1}^{\bm{\lambda}}]]\\
& = Var[r_{t+1}] + \gamma_{t+1}^2 \lambda_{t+1}^2Var[G_{t+1}^{\bm{\lambda}}]\\
& \text{(assuming $r_{t+1}$ \& $G_{t+1}^{\bm{{\bm{\lambda}}}}$ uncorrelated)}
\end{aligned}
\nonumber
\end{equation}

Assuming negligible effects of $\lambda_{t+1}$ on the statistics, \ie{} not taking the partial derivatives of the expectation or the variance, we can obtain the semi-partial derivative 
\begin{equation}
\begin{aligned}
\frac{\partial J_{s_t}(\bm{\lambda})}{\partial \lambda_{t+1}} & \equiv \frac{\partial}{\partial \lambda_{t+1}} \left(\doubleE[(G_t^{\bm{{\bm{\lambda}}}} - \doubleE[G_t])^2] \right)\\
& = \frac{\partial}{\partial \lambda_{t+1}} ((\gamma_{t+1} (1 - \lambda_{t+1})V(s_{t+1})\\
\end{aligned}
\nonumber
\end{equation}
\begin{equation}
\begin{aligned}
& + \gamma_{t+1} \lambda_{t+1} \doubleE[G_{t+1}^{\bm{\lambda}}] - \gamma_{t+1} \doubleE[G_{t+1}])^2\\
& + Var[r_{t+1}]+ \gamma_{t+1}^2 \lambda_{t+1}^2Var[G_{t+1}^{\bm{\lambda}}])\\
& = \gamma_{t+1}^2 [\lambda_{t+1} \left( (V(s_{t+1}) - \doubleE[G_{t+1}^{\bm{{\bm{\lambda}}}}])^2 + Var[G_{t+1}^{\bm{{\bm{\lambda}}}}] \right)\\
& + (\doubleE[G_{t+1}^{\bm{{\bm{\lambda}}}}] - V(s_{t+1}))(\doubleE[G_{t+1}] - V(s_{t+1}))]
\end{aligned}
\nonumber
\end{equation}

The minimizer is achieved by setting the partial derivative $0$.
\end{proof}

This proposition constructs a way to estimate the partial derivative that corresponds to the dimension of $\lambda_{t+1}$ in $\nabla \bm{\lambda}$, if we know or can effectively estimate the statistics of $\doubleE[G_{t+1}]$, $\doubleE[G_{t+1}^{\bm{\lambda}}]$ and $Var[G_{t+1}^{\bm{\lambda}}]$. This proposition also provides the way for finding a whole series of partial derivatives and also naturally yields a multi-step method of approximating the full gradient $\nabla_{\bm{\lambda}} \doubleE[(G_t^{\bm{{\bm{\lambda}}}} - \doubleE[G_t])^2]$. The partial derivative in the proposition is achieved by looking $1$-step into the future. We can also look more steps ahead, and get the partial derivatives \wrt{} $\lambda^{(t+2)}, \cdots$. These partial derivatives can be computed with the help of the auxiliary tasks as well. The more we assemble the partial derivatives, the closer we get to the full gradient. However, in our opinion, $1$-step is still the most preferred not only because it can be obtained online every step without the need of buffers but also for its dominance over other dimensions of $\bm{\lambda}$: the more steps we look into the future, the more the corresponding $\lambda$s of the states are discounted by the earlier $\gamma$s and $\lambda$s. This also enables the computation of the whole gradient if we were to do the adaptations offline, in which case everything would be more precise and easier, though more computationally costly.

It is interesting to observe that the minimizer is a generalization of (\ref{eq:white_argmin}): the minimizer of the greedy target error can be achieved by setting $G_{t+1}^{\bm{\lambda}} = G_{t+1}$. In practice, given an unknown MDP, the distributional information of the targets, \eg{} $\doubleE[G_{t+1}]$, $\doubleE[G_{t+1}^{\bm{\lambda}}]$ and $Var[G_{t+1}^{\bm{\lambda}}]$, can only be estimated. However, such estimation has been proved viable in both offline and online settings of TD($\lambda$) and the variants, using supervised learning and auxiliary tasks using the direct VTD method \cite{sherstan2018directly}, respectively. This means the optimization for the ``true'' target error is as viable as the $\lambda$-greedy method proposed in \cite{white2016greedy}, while it requires more complicated estimations than that for the ``greedy'' target error: we need the estimates of $\doubleE[G_{t+1}]$, $\doubleE[G_{t+1}^{\bm{\lambda}}]$ and $Var[G_{t+1}^{\bm{\lambda}}]$, while for (\ref{eq:white_argmin}) we only need the estimation of $\doubleE[G_{t+1}]$ and $Var[G_{t+1}]$.

The optimization of the true state target error, \ie{} the MSE between $\bm{\lambda}$-return and the true value, together with the auxiliary estimation, brings new challenges: the auxiliary estimates are learnt online and requires the stationarity of the update targets. This means if a $\lambda$ for one state is changed dramatically, the auxiliary estimates of $\doubleE[G_{t+1}^{\bm{\lambda}}]$ and $Var[G_{t+1}^{\bm{\lambda}}]$ will be destroyed, since they depend on each element in ${\bm{\lambda}}$ (whereas in $\lambda$-greedy, the pseudo targets require no $\bm{\lambda}$-controlled distributional information). If we cannot handle such challenge, either we end up with a method that have to wait for some time after some change of $\bm{\lambda}$ or we end up with $\lambda$-greedy, bearing the high bias towards the MC return and disconnection from the overall target error.

Adjusting ${\bm{\lambda}}$ without destroying the auxiliary estimates is a core problem. We tackle such optimization by noticing that the expectation and variance of the update targets are continuous and differentiable \wrt{} $\bm{\lambda}$. Thus, a small change on $\lambda_{t+1}$ only yields a bounded shift of the estimates of the auxiliary tasks.
If we use small enough steps of the estimated gradients to change $\bm{\lambda}$, we can stabilize the auxiliary estimates since they will not deviate far and will be corrected by the TD updates quickly. This method inherits the ideas of trust region methods used in optimizing the dynamic systems.

Combining the approximation of gradient and the decomposed $1$-step optimization method, we now have an online method to optimize $\bm{\lambda}$ to achieve approximate optimization of the overall target error, which we name \algoname{}. This method can be jointly used with value learning, serving as a plugin, to adapt $\bm{\lambda}$ in real-time. Before we present the whole algorithm, we would like to first discuss the properties, potentials as well as limitations of \algoname{}.




\begin{algorithm*}[htbp]
\caption{\algoname{}-assisted Online Policy Evaluation}
\label{alg:MTA_PE}
Initialize weights for the value learner and those for the auxiliary learners that learns $\hat{\doubleE}[G_t]$, $\hat{\doubleE}[G_t^{\bm{\lambda}}]$ and $\hat{V}ar[G_t^{\bm{\lambda}}]$\\
\For{episodes}{
    $\rho_{acc} = 1$; \textcolor{darkgreen}{//initialize cumulative product of importance sampling ratios}\\
    Set traces for value learner and auxiliary learners to be $\bm{0}$;\\
    $\bm{x}_0 = \text{initialize}(\scriptE{})$; \textcolor{darkgreen}{//Initialize the environment $\scriptE{}$ and get the initial feature (observation) $\bm{x}_0$}\\
    
    \While{$t \in \{0, 1, \dots\}$ until terminated}{
    \textcolor{darkgreen}{//INTERACT WITH ENVIRONMENT}\\
    $a_t \sim b(\bm{x}_{t})$; \textcolor{darkgreen}{//sample $a_t$ from behavior policy $b$}\\
    $\rho_t = {\pi(a_{t}, \bm{x}_{t})} / {b(a_{t}, \bm{x}_{t})}$; $\rho_{acc} = \rho_{acc} \cdot \rho_t$; \textcolor{darkgreen}{//get and accumulate importance sampling ratios}\\
    $\bm{x}_{t+1}, \gamma_{t+1} = \text{step}(a_t)$;
    \textcolor{darkgreen}{//take action $a_t$, get feature (observation) $\bm{x}_{t+1}$ and discount factor $\gamma_{t+1}$}\\
    \textcolor{purple}{//AUXILIARY TASKS}\\
    learn $\hat{\doubleE}[G_t]$, $\hat{\doubleE}[G_t^{\bm{\lambda}}]$ and $\hat{V}ar[G_t^{\bm{\lambda}}]$; \textcolor{purple}{//using direct VTD \cite{sherstan2018directly} with trace-based TD methods, \eg, true online GTD($\bm{\lambda}$) \cite{hasselt2014true}}\\
    
    \textcolor{purple}{//APPROXIMATE SGD ON OVERALL TARGET ERROR}\\
    
    $\lambda_{t+1} = \lambda_{t+1} - \kappa \gamma_{t+1}^2 \rho_{acc} \left[\lambda_{t+1} \left( (V(\bm{x}_{t+1}) - \hat{\doubleE}[G_{t+1}^{\bm{\lambda}}])^2 + \hat{V}ar[G_{t+1}^{\bm{\lambda}}] \right) + (\hat{\doubleE}[G_{t+1}^{\bm{\lambda}}] - V(\bm{x}_{t+1}))(\hat{\doubleE}[G_{t+1}] - V(\bm{x}_{t+1}))\right]$; \textcolor{purple}{// change $\lambda_{t+2}, \cdots$ when using multi-step approximation of the gradient}\\
    
    \textcolor{darkgreen}{//LEARN VALUE}\\
    learn $V(\bm{x}_{t})$ using a trace-based TD method;
}
}
\end{algorithm*}

\section{Discussions and Insights}

\subsection{Hyperparameter Search}

\algoname{} trades the search for $\lambda$ with $\kappa$, the step size of \algoname{}-optimization. However, $\kappa$ gives the algorithm the ability to have state-based $\lambda$s: state or feature (observation) based $\bm{\lambda}$ can lead to better convergence compared to fixing $\lambda$ for all states. Such potential may never be achieved by searching a fixed $\lambda$. Let us consider the tabular case, where the search for constant $\bm{\lambda} = {\lambda} \bm{1}$ is equivalent to searching along the diagonal direction inside a $|\scriptS|$-dimensional unit box ${[0, 1]}^{|\scriptS|}$. By replacing $\lambda$ with $\kappa$, we extend the search direction of $\bm{\lambda}$ into the whole unit box. The new degrees of freedom are crucial to the performance.

\subsection{Reliance on Auxiliary Tasks}\label{subsection:reliance}

\algoname{} updates assume that $\hat{\doubleE}[G_t]$, $\hat{\doubleE}[G_t^{\bm{\lambda}}]$ and $\hat{V}ar[G_t^{\bm{\lambda}}]$ can be well estimated by the auxiliary tasks. This is very similar to the idea of actor changing the policy upon the estimation of the values of the critic in the actor-critic methods. To implement this, we can add a buffer period for the estimates to be stable before doing any adaptation; Additionally, we should set the learning rates of the auxiliary learners higher than the value learner \st{} the auxiliary tasks are learnt faster, resembling the guidelines for setting learning rates of actor-critic. 
With the buffer period, we can also view \algoname{} as approximately equivalent to offline hyperparameter search of $\bm{\lambda}$, where with \algoname{} we first reach a relatively stable accuracy and then adjust $\bm{\lambda}$ to slowly slide to fixed points with lower errors. Also, \algoname{} is compatible with fancier settings of learning rate, since the meta-adaptation is independent of its values.

\subsection{Function Approximation}
With function approximation, the meta-learning of $\bm{\lambda}$-greedy cannot make use of the state features directly but through the bottlenecks of the estimates. Whereas in \algoname{}, $\bm{\lambda}$ can be parameterized and optimized with gradient descent. This enables better generalization and can be effective when the state features contain rich information (good potential to be used with deep neural networks). This is to be demonstrated in the experiments.





\subsection{From Prediction to Control}

Within the control tasks where the quality of prediction is crucial to the policy improvement, it is viable to apply \algoname{} to enhance the policy evaluation process. \algoname{} is a trust region method, which requires the policy to be also changing smoothly, \st{} the shift of values can be bounded. This constraint leads us naturally to the actor-critic architectures, where the value estimates can be used to improve a continuously changed parametric policy. We provide the pseudocode of \algoname{}-assisted actor-critic control in Algorithm \ref{alg:MTA_AC}.

\begin{algorithm*}[htbp]
\caption{\algoname{}-assisted Online Actor-Critic}
\label{alg:MTA_AC}
Initialize weights for the value learner and those for the auxiliary learners that learns $\hat{\doubleE}[G_t]$, $\hat{\doubleE}[G_t^{\bm{\lambda}}]$ and $\hat{V}ar[G_t^{\bm{\lambda}}]$\\
Initialize parameterized policies $\pi(\cdot | \theta_\pi)$ and $b(\cdot | \theta_b)$;\\
\For{episodes}{
    Set traces for value learner and auxiliary learners to be $\bm{0}$;\\
    $\bm{x}_0 = \text{initialize}(\scriptE{})$;\\
    
    \While{$t \in \{0, 1, \dots\}$ until terminated}{
    $a_t \sim b(\bm{x}_{t})$; $\rho_t = {\pi(a_{t}, \bm{x}_{t})} / {b(a_{t}, \bm{x}_{t})}$; $\rho_{acc} = \rho_{acc} \cdot \rho_t$;\\
    $\bm{x}_{t+1}, \gamma_{t+1} = \text{step}(a_t)$;\\
    \textcolor{purple}{//AUXILIARY TASKS and SGD ON OVERALL TARGET ERROR}\\
    learn $\hat{\doubleE}[G_t]$, $\hat{\doubleE}[G_t^{\bm{\lambda}}]$ and $\hat{V}ar[G_t^{\bm{\lambda}}]$;\\
    $\lambda_{t+1} = \lambda_{t+1} - \kappa \gamma_{t+1}^2 \rho_{acc} \left[\lambda_{t+1} \left( (V(\bm{x}_{t+1}) - \doubleE[G_{t+1}^{\bm{\lambda}}])^2 + Var[G_{t+1}^{\bm{\lambda}}] \right) + (\doubleE[G_{t+1}^{\bm{\lambda}}] - V(\bm{x}_{t+1}))(\doubleE[G_{t+1}] - V(\bm{x}_{t+1}))\right]$;\\
    
    learn $V(\bm{x}_{t})$ using a trace-based TD method;\\
    
    \textcolor{red}{//LEARN POLICY}\\
    One (small) step of policy gradient (actor-critic) on $\theta_\pi$;\\
}
}
\end{algorithm*}


\subsection{Overview \& Limitations}
\algoname{} can be injected as a plugin for accelerating (improving the sample efficiency of) TD-based policy evaluation processes. This is illustrated in Figure \ref{fig:\algoname{}}: by adding $3$ auxiliary learners to the system, better feature-based $\bm{\lambda}$ can be achieved using only the existing information. In Algorithm \ref{alg:MTA_PE}, \algoname{} is injected to a TD-based baseline as the additional two lines that are with purple comments. The first is for the $3$ auxiliary learners estimating the statistics with trace-based online updates, using either VTD \cite{white2016greedy} or DVTD \cite{sherstan2018directly} methods. The second is the $1$-step update approximating the $1$-step gradient descent. The injected process uses additional computational costs approximately $3$ times that of the baseline yet incurring no higher order complexities. In Algorithm \ref{alg:MTA_AC}, \algoname{} is injected to assist the critic update for value estimation with the same mechanisms.

Meaningful as it is, \algoname{} has its limitations. First, though the trust region optimization enabled the optimization of joint error, it also brought trouble: the adaptation is bound to be slow. In prediction, given the changing $\bm{V}$, the ``optimal'' $\bm{\lambda}$ also changes, presumably fast. Therefore, \algoname{} may not able to catch up with the need for fast adaptation, even if it is always chasing the ``optimal'' $\bm{\lambda}$; Second, being a gradient method, the stepsize parameter $\kappa$ is inevitably sensitive to the feature structures. For example, if the features are large in norm then $\kappa$ must be set tiny; Third, when used with actor-critic control, it further requires that the policy to be changing slowly. We will leave these problems for future research.

\section{Experiments}\label{sec:experiments}


We examine the empirical behavior of \algoname{} by comparing it to the baselines true online TD($\lambda$) \cite{seijen2015true} or true online GTD($\lambda$) \cite{hasselt2014true} as well as the $\lambda$-greedy method \cite{white2016greedy}\footnote{Source code is available at: \url{https://github.com/PwnerHarry/META}}. For all sets of tests, $\bm{\lambda}$ start adapting from $\bm{1}$\footnote{Such setting is enabled by using $\lambda(\bm{x}) = 1 - \bm{w}_\lambda^T\bm{x}$ as the function approximator of the parametric $\lambda$ and the weights initialized as $\bm{0}$.}, which is the same as $\lambda$-greedy \cite{white2016greedy}.

\begin{figure*}
\centering

\subfloat[RingWorld, $\gamma = 0.95$]{
\captionsetup{justification = centering}
\includegraphics[width=0.3\textwidth]{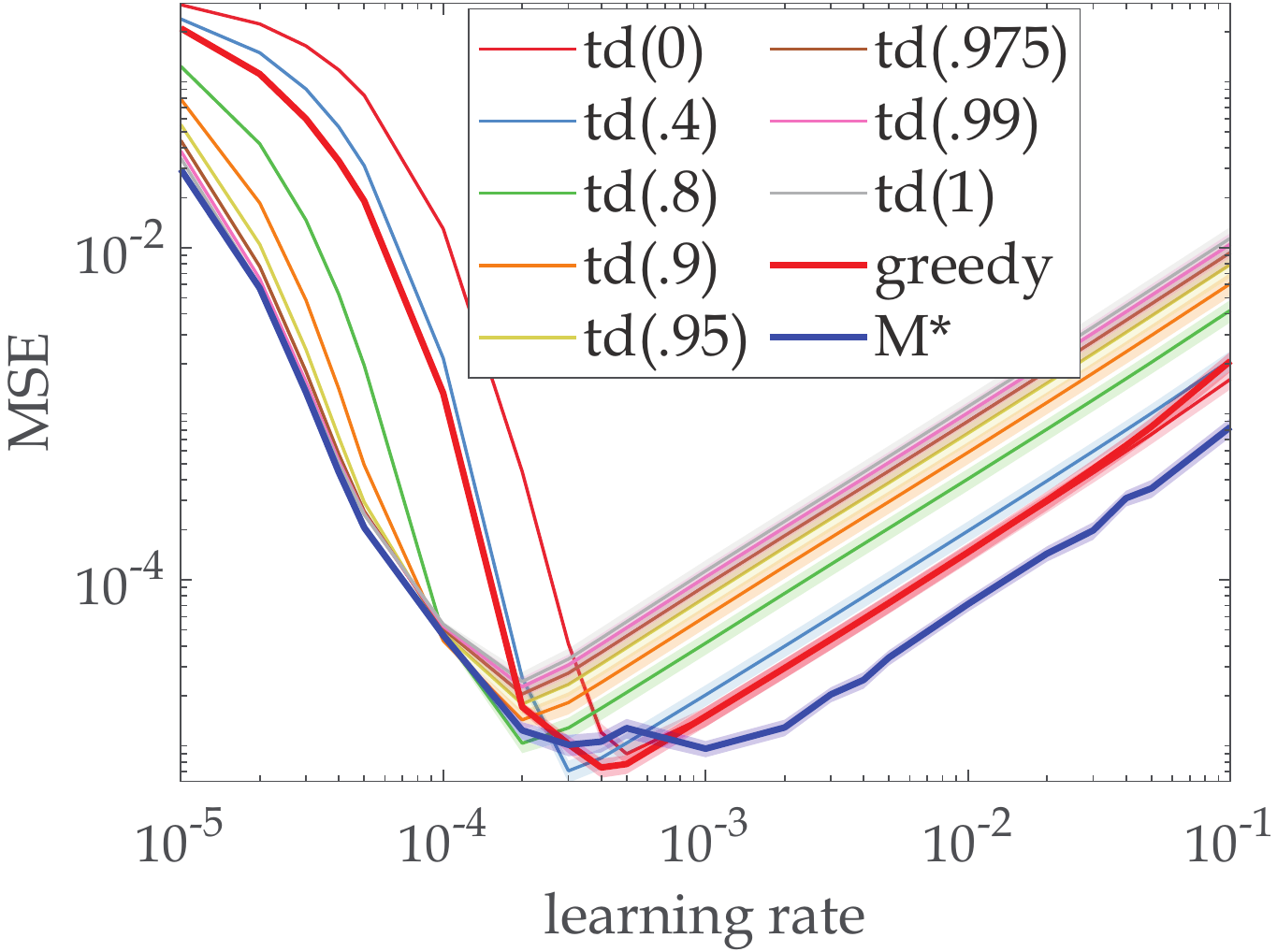}}
\hfill
\subfloat[FrozenLake, $\gamma = 0.95$]{
\captionsetup{justification = centering}
\includegraphics[width=0.3\textwidth]{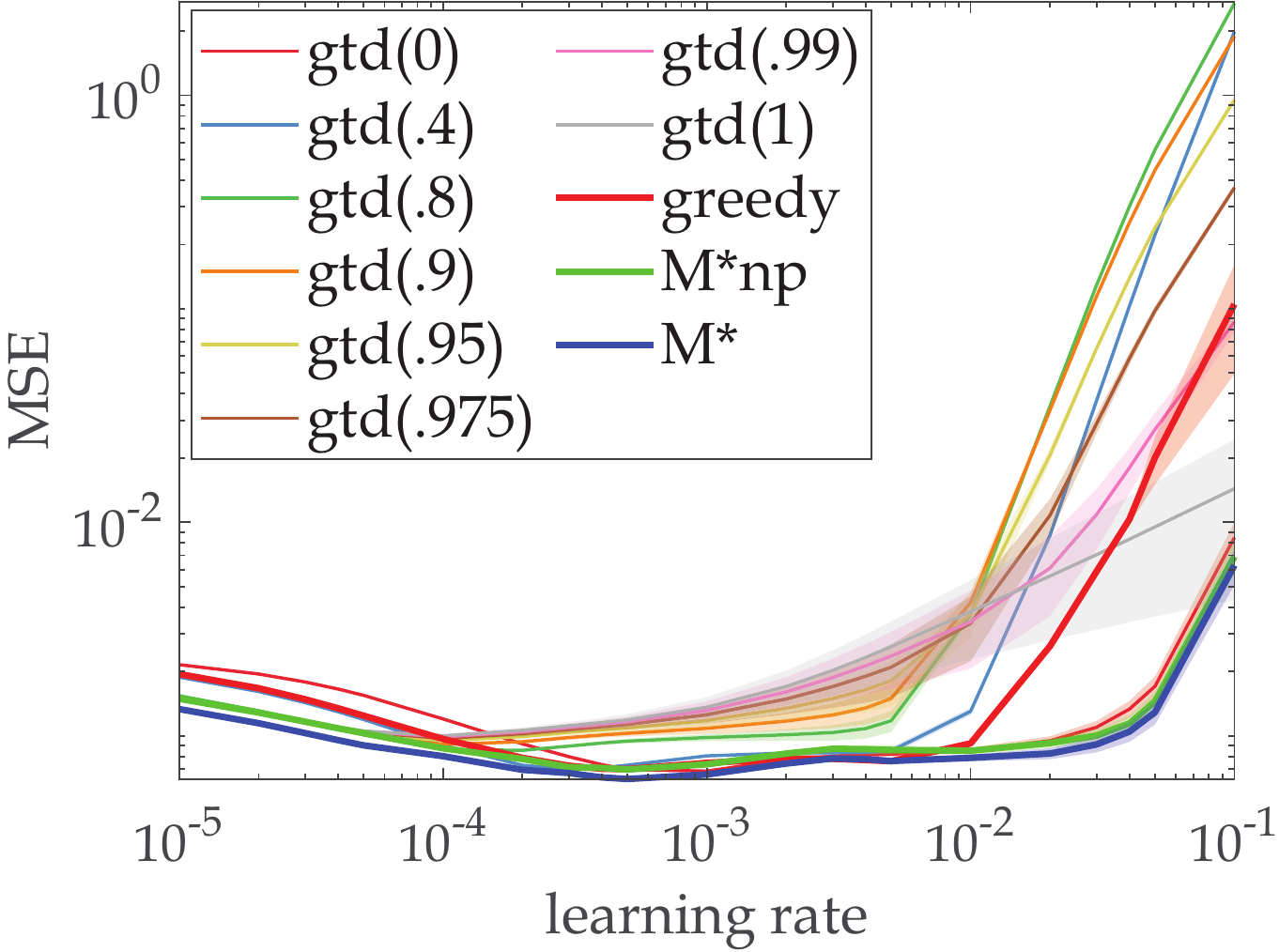}}
\hfill
\subfloat[MountainCar, $\gamma = 1$, $\eta = 1$]{
\captionsetup{justification = centering}
\includegraphics[width=0.3\textwidth]{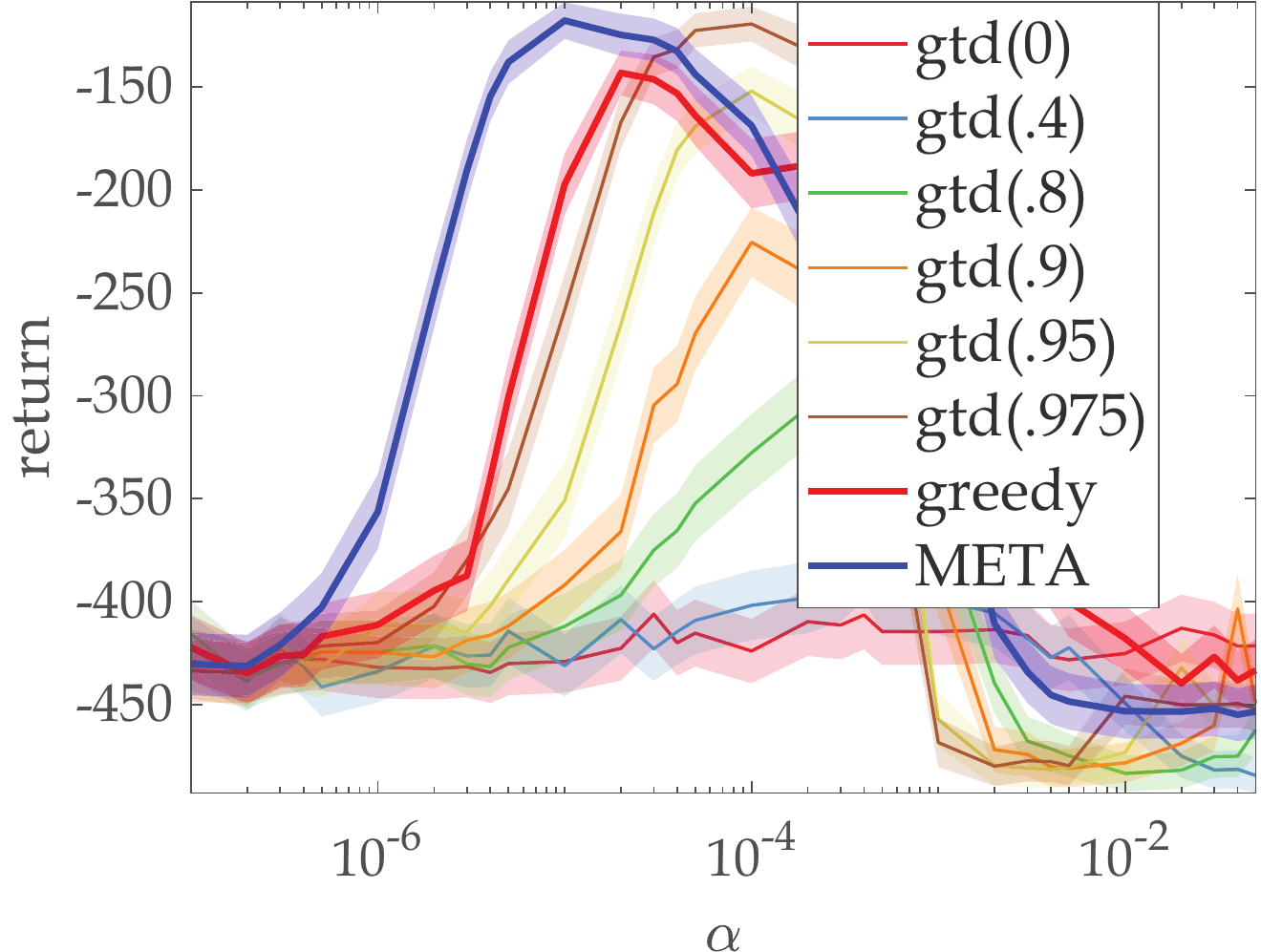}}

\subfloat[RingWorld, $\alpha = 0.01$, $\kappa = 0.01$]{
\captionsetup{justification = centering}
\includegraphics[width=0.3\textwidth]{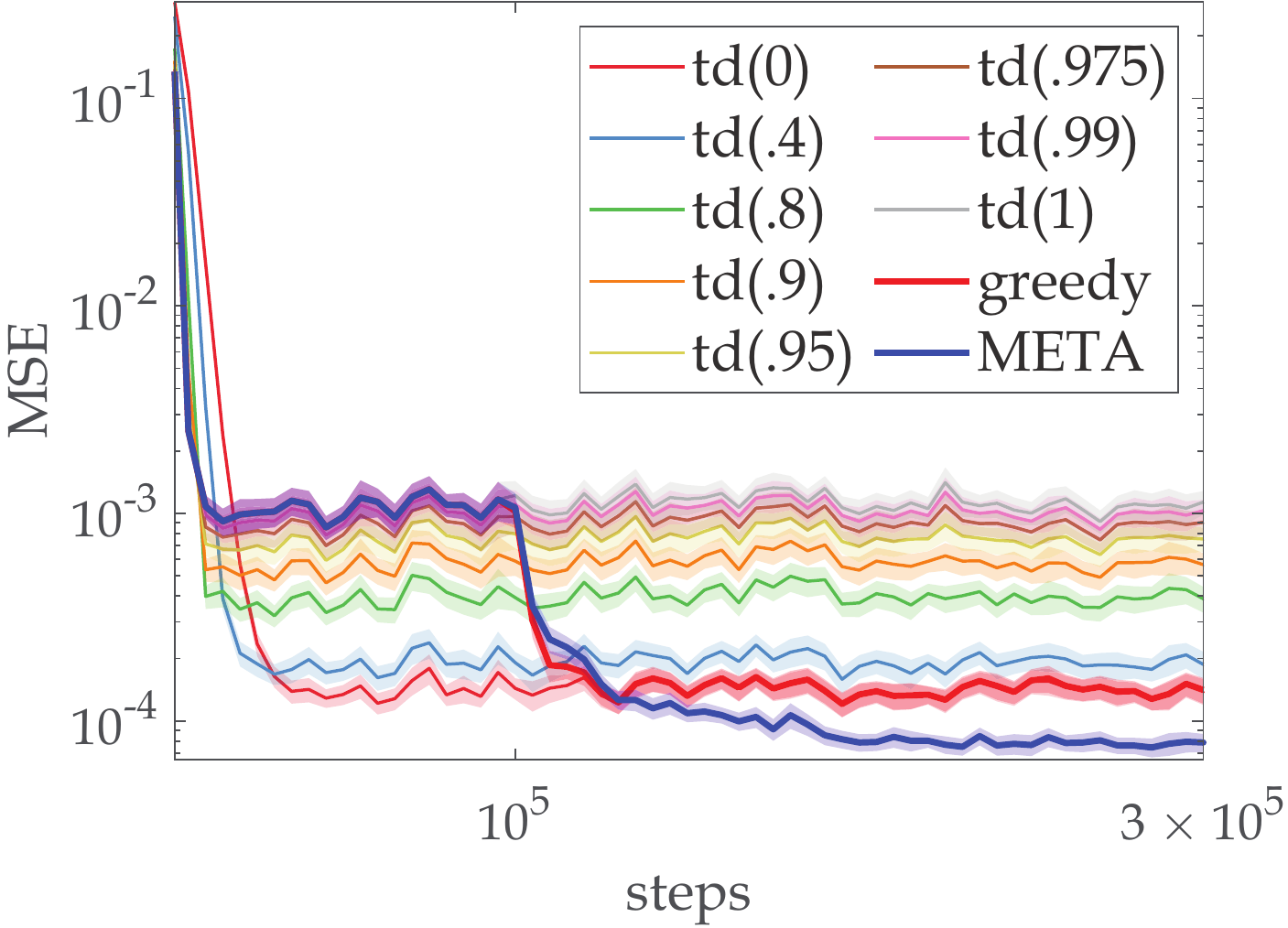}}
\hfill
\subfloat[FrozenLake, $\alpha = \beta = 0.0001$, $\kappa = 10^{-5}$, $\kappa_{\text{np}} = 10^{-4}$]{
\captionsetup{justification = centering}
\includegraphics[width=0.3\textwidth]{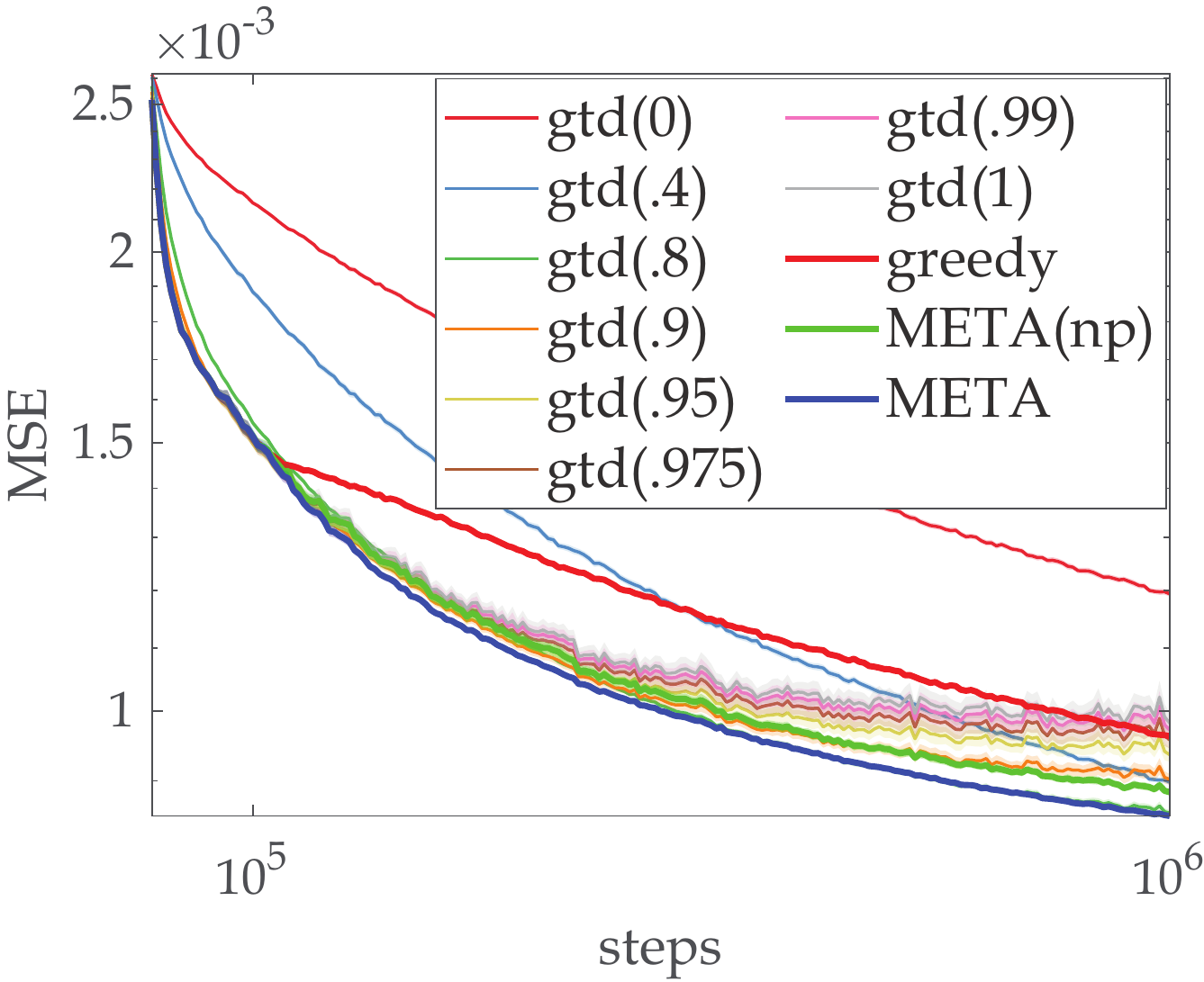}}
\hfill
\subfloat[MountainCar, $\alpha = \beta = 10^{-5}$, $\eta =1$, $\kappa = 10^{-5}$]{
\captionsetup{justification = centering}
\includegraphics[width=0.3\textwidth]{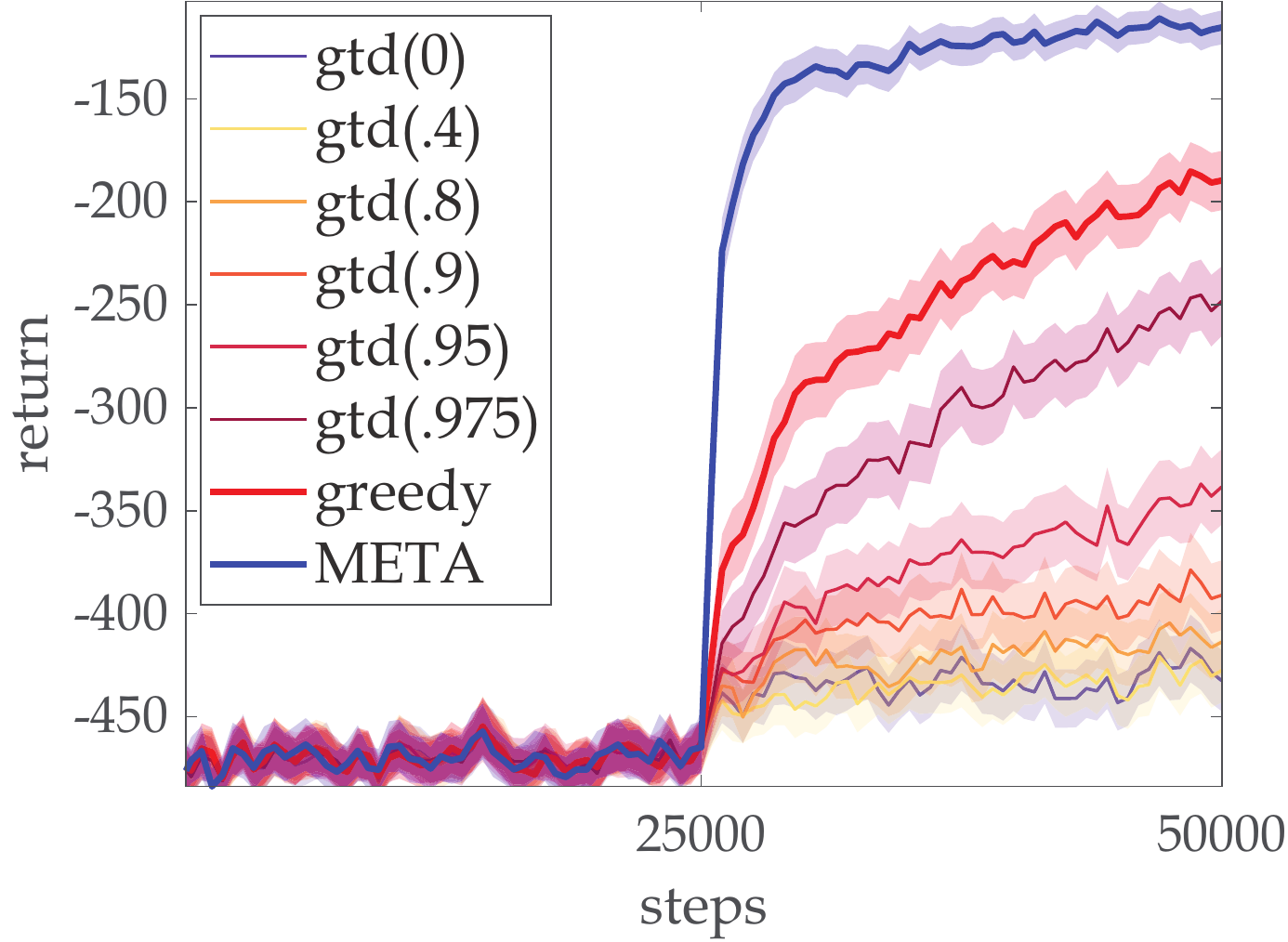}}

\caption{U-shaped curves and learning curves for \algoname{}, $\lambda$-greedy and the baselines on RingWorld, FrozenLake and MountainCar. For (a), (b) and (c), $x$-axes represent the values of the learning rate $\alpha$ for prediction (or the critic), while $y$-axes represent the overall value error for RingWorld and FrozenLake, or the cumulative discounted return for MountainCar. Each point in the graphs is featured with the mean (solid) and standard deviation (shaded) collected from $240$ independent runs, with $10^{6}$ steps for prediction and $50000$ steps for control. The blue curves, either with legend ``META'' or ``M*'', represent the results of \algoname{}; For (d), (e), and (f), the $x$-axes represent the steps. We choose one representative case for each of the corresponding U-shaped curve for better demonstration of the empirical performance of \algoname{}. In these learning curves, the best known hyperparameters are used. The buffer periods are $10^{5}$ steps ($10\%$) for prediction and $25000$ steps ($50\%$) for control, respectively.} 
\label{fig:U}
\end{figure*}

\subsection{RingWorld: Tabular-Case Prediction}
This set of experiments focuses on a low-variance environment, the $11$-state ``ringworld'' \cite{white2016greedy}, in which the agent move either left or right in a ring of states. The state transitions are deterministic and rewards only appear in the terminal states. In this set of experiments, we stick to the tabular setting and use true online TD($\bm{\lambda}$) \cite{hasselt2014true} as the learner\footnote{We prefer true online algorithms since they achieve the exact equivalence of the bi-directional view of $\lambda$-returns.}, for the value estimate as well as all the auxiliary estimates. As discussed in \ref{subsection:reliance}, for the accuracy of the auxiliary learners, we double their learning rate \st{} they can adapt to the changes of the estimates faster. We select a pair of behavior-target policies: the behavior policy goes left with $0.4$ probability while the target policy goes with $0.35$. The baseline true online TD has $2$ hyperparameters ($\alpha$ \& $\lambda$) and so does \algoname{} ($\alpha$ \& $\kappa$), excluding those for the auxiliary learners. For these two methods, we test them on grids of hyperparameter pairs. More specifically, for the baseline true online TD, we test it on $\langle \alpha, \lambda \rangle \in \{ 10^{-5}, \dots, 5\times10^{-5}, 10^{-4}, \dots, 5\times10^{-4}, \dots, 5\times10^{-2}, 10^{-1} \} \times \{0, 0.4, 0.8, 0.9, 0.95, 0.975, 0.99, 1\}$ while for \algoname{}, $\langle \alpha, \kappa \rangle \in \{ 10^{-5}, \dots, 5\times10^{-5}, 10^{-4}, \dots, 5\times10^{-4}, \dots, 5\times10^{-2}, 10^{-1} \} \times \{10^{-7}, \dots, 10^{-1}\}$. The results are presented as the U-shaped curves in Figure \ref{fig:U} (a), in which we demonstrate the curves of the baseline under different $\lambda$s and the best performance that \algoname{} could get under each learning rate.
\par
The best performance of fine-tuned baselines can be extracted from the figures by combining the lowest points of the set of the baseline curves under different $\lambda$s. Fine-tuned \algoname{} provides better performance, especially when the learning rate is relatively high. We can say that once \algoname{} is fine-tuned, it provides significantly better performance that the baseline algorithm cannot possibly achieve since it meta-learns state-based $\bm{\lambda}$ that goes beyond the scope of the optimization of the baseline. Such results can also be interpreted as \algoname{} being less sensitive to the learning rate than the baseline true online TD.

\subsection{FrozenLake: Feature-based Prediction}

This set of experiments features on a high-variance environment, the ``4x4'' FrozenLake, in which the agent seeks to fetch the frisbee back on a frozen lake surface with holes from the northwest to the southeast and the transitions are noisy. There are $4$ actions, each representing taking $1$-step towards $1$ of the $4$ directions. We craft a behavior policy that takes $4$ actions with equal probabilities and a target policy that has $0.3$ probability for going south or east, $0.2$ for going north or west. We use the linear function approximation based true online GTD($\lambda$), with a discrete tile coding ($4$ tiles, $4$ offsets). For the 2nd learning rate $\beta$ introduced in true online GTD($\lambda$), we set them to be the same as $\alpha$ (for the value learners as well as the auxiliary learners in all the compared algorithms). Additionally, we remove the parametric setting of $\bm{\lambda}$ to get a method as ``\algoname{}(np)'' to demonstrate the potentials of a parametric feature (observation) based $\bm{\lambda}$. The U-shaped curves, obtained using the exact same settings as in RingWorld, are provided in Figure \ref{fig:U} (b).
\par
We observe similar patterns as the 1st set of experiments. We can see that the generalization provided by the parametric $\bm{\lambda}$ is beneficial, as in (b) we observe generally better performance and in (e) we see that a parametric $\lambda$ has better sample efficiency, comparing with ``\algoname{}(np)''. This suggests that using parametric $\lambda$ in environments with relatively smooth dynamics would be generally beneficial for sample efficiency.

\subsection{MountainCar: Actor-Critic Control}

In this set of experiments we investigate the use of \algoname{} to assist on-policy actor-critic control on a noisy version of the environment MountainCar with tile-coded state features. We use a softmax policy parameterized by a $|A| \times D$ matrix, where $D$ is the dimension of the state features with also true online GTD($\lambda$) as the learners (critics). This time, the U-shaped curves presented in Figure \ref{fig:U}(c) show performance better than the baselines yet significantly better than $\lambda$-greedy assisted actor-critic.

In this set of experiments we intentionally set the stepsize of the gradient ascent of the policy to be high ($\eta = 1$) to emphasize the quality of policy evaluation. However, typically in actor-critic we keep $\eta$ small. In these cases, the assistance of \algoname{} is expected to be greatly undermined: the maximization of returns cares more about the actions chosen rather than the accuracy of the value estimates. Enhancing the policy evaluation quality may not be sufficient for increasing the sample efficiency of control problems.

From the curves we can see the most significant improvements are shown when the learning rate of the critic is small. Typically in actor-critic, we set the learning rate of the critic to be higher than the actor to improve the quality of the update of the actor. \algoname{} alleviates the requirement for such setting (or we could say a kind of sensitivity) by boosting the sample efficiency of the critic.

\subsection{Technical Details}
\subsubsection{Environments}
The RingWorld environment is reproduced as described in \cite{white2016greedy}. Due to limitations of understanding, we cannot see the difference between it and a random walk environment with the rewards on the two tails. RingWorld is described as a symmetric ring of the states with the starting state at the top-middle, for which we think the number of states should be odd. However, the authors claimed that they experimented with $10$-state and $50$-state instances. We instead used the $11$-state instance.

We removed the episode length limit of the FrozenLake environment (for the environment to be solvable by dynamic programming). It is modified based on the Gym environment with the same name. We have used the instance of ``4x4'', \ie{} with $16$ states.

The episode length limit of MountainCar is also removed. We also added noise to the state transitions: actions will be randomized at $20\%$ probability. The noise is to prevent the cases in which $\lambda = 1$ yields the best performance (to prevent \algoname{} from using extremely small $\kappa$'s to get good performance). Additionally, due to the poor exploration of the softmax policy, we extended the starting location to be uniformly anywhere from the left to right on the slopes.

\subsubsection{State Features}
For RingWorld, we used onehot encoding to get equivalence to tabular case; For FrozenLake, we used a discrete variant of tile coding, for which there are $4$ tilings, with each tile covering one grid as well as symmetric and even offset; For MountainCar, we adopted the roughly the same setting as Chap. 10.1 pp. 245 in \cite{sutton2018reinforcement}, except that we used ordinary symmetric and even offset instead of the asymmetric offset.

\subsubsection{About ${\bm{\lambda}}$-greedy}
We have replaced VTD \cite{white2016greedy} with direct VTD \cite{sherstan2018directly}. This modification is expected only to improve the stability, without touching the core mechanisms of $\lambda$-greedy \cite{white2016greedy}.

The target used in Whites' \cite{white2016greedy} is biased toward $\lambda = 1$, as the $\lambda$'s into the future are assumed to be $1$. Thus we do not think it is helpful to conduct tests on environments with very low variance. This is the reason why we have changed the policies to less greedy.

\subsubsection{Learning Rate and Buffer Period}
The learning rates of the auxiliary learners are set to be twice of the value learner. These settings were not considered in \cite{white2016greedy}, in which there were no buffer period and identical learning rates were used for all learners; For the control task of MountainCar, $\lambda$-greedy and \algoname{} will both perform badly without these additional settings, since they are adapting $\bm{\lambda}$ based on untrustworthy estimates.

\subsubsection{Details for Non-Parametric $\lambda(\cdot)$}
To disable the generalization of the parametric $\bm{{\bm{\lambda}}}$ for ``\algoname{}(np)'', we replaced the feature vectors for each state with onehot-encoded features.

\subsubsection{More Policies for Prediction}
For RingWorld, we have done $6$ different behavior-target policy pairs (3 on-policy \& 3 off-policy). The off-policy pair that we have shown in the manuscript shares the same patterns as the rest of the pairs. The accuracy improvement brought by \algoname{} is significant across these pairs of policies; For FrozenLake, we have done two pairs of policies (on- and off-policy). We observe the same pattern as in the RingWorld tests.

\subsubsection{Implementation of \algoname{}}
Due to the estimation instability, updates could bring state $\lambda$ values outside $[0, 1]$. Whenever such kind of update is detected, it will be canceled.

\section{Conclusion and Future Work}
In this paper, we derived a general method \algoname{} for boosting the sample efficiency of TD prediction, by approximately optimizing the overall target error, using meta-learning of state dependent $\lambda$s. In the experiments, \algoname{} demonstrates promising performance as a way to accelerate learning.

In the future, we aim to benchmark the approach in more environments, and in general. We would also like to study further the issue of improving optimizers for RL specifically.

\subsection*{Acknowledgements}

Funding for this research was provided in part by NSERC, through Discovery grants for Prof. Chang and Precup, and CIFAR, through a CCAI chair to Prof. Precup. We are grateful to Compute Canada for providing a shared cluster for experimentation.

\clearpage
\bibliographystyle{ACM-Reference-Format}
\bibliography{references.bib}

\end{document}

%% file: tab_notations.tex
\begin{table*}[htbp]
\small
\centering
\caption{Notations}
\label{tab:notations}
\begin{tabular}{cm{0.7\textwidth}}
    \toprule
    \toprule
    Notation & Meaning \\
    \midrule
    $\bm{x}_{t}$     & Feature vector or observation for the state $s_t$ met at time-step $t$. \\
    $V_\pi(s_t)$ or $V(s_t)$  & Estimated value function or estimated values for $s_{t}$. \\
    $v_\pi(s_t)$, $v(s_t)$ or $\doubleE[G_t]$  & True expectation of $G_t$ for $s_{t}$, also recognized as the true value.\\
    $\bm{\lambda}$ & Enumeration vector of $\lambda$'s for all states.\\
    $G_t$ & Cumulative discounted return since time-step $t$. \\
    $\rho_t$    & Importance sampling ratio for the action $a_t$ taken at time-step $t$.\\
    $\gamma_t$  & Discount factor for returns after meeting the state $s_t$ at time-step $t$ \cite{sutton2011horde}. \\
    $d_\pi(s)$ & The frequency of meeting the state $s$ among all states, when carrying out policy $\pi$ infinitely in the environment. \\
    \bottomrule
    \bottomrule
\end{tabular}%
\end{table*}